\documentclass[sigconf]{acmart}

\AtBeginDocument{%
  }

\setcopyright{acmlicensed}
\copyrightyear{2026}
\acmYear{2026}
\acmDOI{XXXXXXX.XXXXXXX}
\acmConference[MM '26]{The 34th ACM International Conference on Multimedia}
  {November 10--14, 2026}{Rio de Janeiro, Brazil}
\acmBooktitle{Proceedings of the 34th ACM International Conference on
  Multimedia (MM '26), November 10--14, 2026, Rio de Janeiro, Brazil}
\acmISBN{978-1-4503-XXXX-X/2026/11}

\acmSubmissionID{2254}

\usepackage{algorithm}
\usepackage{algorithmic}
\usepackage{booktabs}
\usepackage{multirow}
\usepackage{tabularx, booktabs} 
\usepackage{threeparttable}
\usepackage{diagbox}  
\usepackage{amsmath}
\usepackage{amsthm, mathtools} 
\usepackage{subcaption}
\usepackage[percent]{overpic}
\usepackage{pifont}
\usepackage{xcolor}  
\usepackage{thmtools} 
\usepackage{adjustbox}
\usepackage[table]{xcolor}  
\definecolor{midgray}{HTML}{E6E6E6}  
\definecolor{EC805A}{HTML}{EC805A}
\definecolor{4575b4}{HTML}{4575b4} 
\usepackage{graphicx}
\usepackage{subcaption}
\usepackage{placeins}
\usepackage{pifont}
\usepackage{contour}            
\definecolor{staryellow}{RGB}{255,205,0}
\contourlength{0.12ex}          

\newcolumntype{Y}{>{\centering\arraybackslash}X}
\newcolumntype{L}{>{\raggedright\arraybackslash}X}
\definecolor{MorandiBlue}{HTML}{2A55A4}
\definecolor{vpdr}{HTML}{D7F6FF}  
\usepackage{diagbox}  

\definecolor{lightblue}{HTML}{D7F6FF} 
\definecolor{midgray}{HTML}{E6E6E6}  
\definecolor{orange}{HTML}{EC805A} 

\usepackage{tikz}
\usepackage[utf8]{inputenc}

\usepackage{enumitem}
\usepackage{float}

\newcommand{\blackcircnum}[1]{%
  \tikz[baseline=-0.75ex]{
    \node[
      shape=circle,
      draw=black,
      fill=black,
      text=white,
      inner sep=0.5pt,
      font=\scriptsize\bfseries,
      minimum size=0.8em
    ] (char) {#1};
  }%
}

\makeatletter

\begin{document}

\title[Towards Privacy-Preserving Federated Prompt Tuning under Data Heterogeneity: \\ A Subspace-Decomposed Expert Approach]
{Towards Privacy-Preserving Federated Prompt Tuning under Data Heterogeneity: A Subspace-Decomposed Expert Approach}
%
\author{Yuhua Wang}
\affiliation{%
  \institution{School of Artificial Intelligence,
  Beihang University}
  \city{Beijing}
  \country{China}}
  \email{yuhuawang@buaa.edu.cn}

\author{Xiaodong Li}
\affiliation{%
  \institution{Center for the Applied Statistics,
  School of Statistics,
  Renmin University of China}
  \city{Beijing}
  \country{China}}

\author{Yihao Guo}
\affiliation{%
  \institution{School of Computer Science \& Technology,  
  Beijing Jiaotong University}
  \city{Beijing}
  \country{China}}

\author{Yuxiang Jia}
\affiliation{%
  \institution{School of Software, 
  Beihang University}
  \city{Beijing}
  \country{China}}

\author{Qinnan Zhang}
\authornote{Corresponding authors. This work was supported by Beijing Advanced Innovation Center for Future Blockchain and Privacy Computing.}
\affiliation{%
  \institution{School of Artificial Intelligence,\\
  Beihang University}
  \city{Beijing}
  \country{China}}

\author{Yifan Sun}
\authornotemark[1]
\affiliation{%
  \institution{Center for the Applied
Statistics, School of Statistics, Renmin University of China}
  \city{Beijing}
  \country{China}}

\author{Hainan Zhang}
\affiliation{%
  \institution{School of Artificial Intelligence,\\
  Beihang University}
  \city{Beijing}
  \country{China}}

\author{Yongxin Tong}
\affiliation{%
  \institution{School of Computer Science and Engineering, Beihang University}
  \city{Beijing}
  \country{China}}

\author{Zhiming Zheng}
\affiliation{%
  \institution{School of Artificial Intelligence, Beihang University}
  \city{Beijing}
  \country{China}}

\renewcommand{\shortauthors}{Wang et al.}

\begin{abstract}
Federated prompt tuning (FPT) enables collaborative adaptation of vision--language models (VLMs) using lightweight prompts. Existing methods often address heterogeneity and privacy through a split-prompt design under local differential privacy (DP), combining a shared prompt for global transfer with private prompts for local adaptation. However, a single shared prompt may over-smooth diverse transferable knowledge, weakening the balance between personalization and generalization. Multi-expert prompts (MEPs) can better capture this diversity, but enlarge the communicated space, increasing DP noise and communication cost while making robust expert composition more difficult.
We propose \textbf{FedSEPT}, a privacy-preserving \textbf{Fed}erated \textbf{S}ubspace-decomposed \textbf{E}xpert \textbf{P}rompt \textbf{T}uning. 
Specifically, we employ Subspace-decomposed Expert Modeling (SEM) to parameterize multiple prompt experts with shared low-rank factors, a fixed public basis, and private residuals, thereby confining communication and DP perturbation to a compact factor space while enabling direct server aggregation in a common coordinate system.
We further design Instance-aware Expert Fusion (IEF), which adaptively combines semantically complementary experts via on-device routing and performs efficient logit-level fusion using cached expert-specific text features. 
Extensive experiments on 11 heterogeneous benchmarks show that, under the same privacy constraints, FedSEPT achieves a better trade-off between local adaptation and global generalization than strong baselines.
\end{abstract}




\maketitle

\begin{figure}[t]
    \centering
    \includegraphics[width=0.95\linewidth]{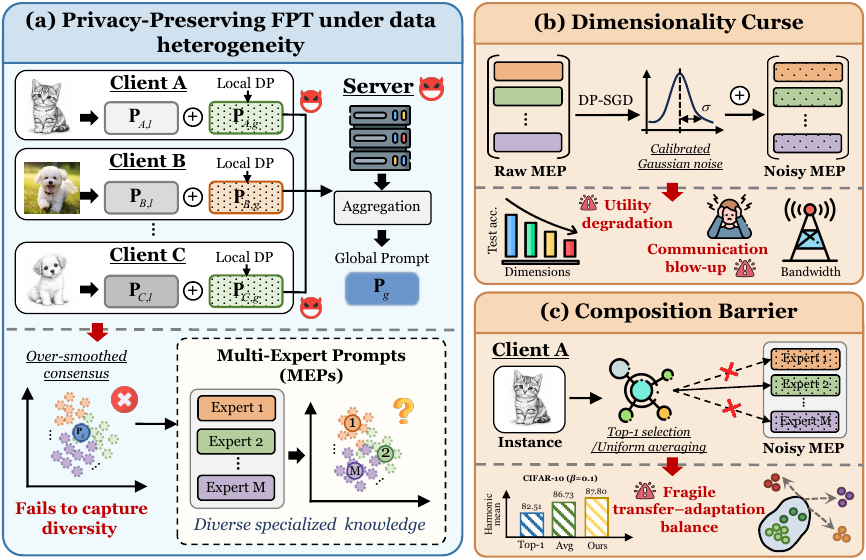}
    \vspace{-2mm} 
    \caption{\textbf{Motivation}. 
    (a) Privacy-Preserving FPT suffers from over-smoothed transferable knowledge under data heterogeneity. 
    Transitioning to MEPs raises: 
    (b) the \textbf{Dimensionality Curse}, where DP noise and communication costs scale with the number of expert prompts; and 
    (c) the \textbf{Composition Barrier}, where noisy expert prompts must be effectively composed for better transfer–adaptation balance.}
    \label{fig:motivation}
\end{figure} 

\section{Introduction}
\label{sec:intro} 
Prompt tuning has been established as a standard paradigm for adapting large vision–language models (VLMs) by optimizing only lightweight context vectors while freezing the backbone~\cite{radford2021clip,zhou2022learning,zhou2022conditional,yao2023visual}.
When combined with federated learning (FL)~\cite{mcmahan2017fedavg}, it gives rise to federated prompt tuning (FPT), which enables multiple clients to collaboratively adapt a shared foundation model without centralizing raw local data~\cite{zhao2023fedprompt,guo2023promptfl}. 
By communicating only prompt-related updates, FPT offers a parameter-efficient path for deploying VLMs in distributed environments such as mobile and edge systems.  
Despite its promise, practical FPT must confront two fundamental problems in real-world FL systems. 
On the one hand, client data often exhibit statistical heterogeneity, with substantial variation in label distributions, visual domains, and semantic concepts~\cite{kairouz2021advances,wang2026taming}. Such non-IID characteristics render a monolithic global adaptation suboptimal for individual local contexts.
On the other hand, prompt updates can leak sensitive information, and an honest-but-curious or untrusted server may infer private client attributes through model updates via attacks such as membership inference~\cite{nasr2019comprehensive,nguyen2026leveraging} and gradient inversion~\cite{geiping2020inverting,Sami_2025_CVPR}.
Hence, \emph{privacy-preserving FPT under data heterogeneity} emerges as a critical desideratum: the learned prompts should not only adapt to diverse local distributions, but also preserve transferable knowledge across clients, while the shared updates remain rigorously protected by formal privacy guarantees.

To address heterogeneity, existing FPT methods commonly adopt a personalized split-prompt paradigm~\cite{yang2023pFedPG,guo2023pfedprompt,cui2024fedpgp,li2024fedotp,fang2025fedpha,zheng2025feddeap}, decomposing prompts into a shared global and a private local component. 
The shared part captures transferable knowledge across clients, whereas the private part remains on-device for local specialization. 
To protect the shared updates, local differential privacy (DP)~\cite{dwork2006differential} is then enforced by perturbing transmitted updates before communication, often via differentially private stochastic gradient descent (DP-SGD)~\cite{abadi2016deep}, which clips per-example gradients and injects calibrated Gaussian noise~\cite{wei2021user,lee2025fedsvd}. 
While this paradigm provides a practical foundation for secure and heterogeneous collaboration, most existing frameworks still rely on a monolithic shared prompt to encode all transferable knowledge.
This design becomes restrictive under severe heterogeneity, where clients often cluster around diverse latent semantics or visual regularities~\cite{ye2023,DiPrompT2024}.
Compressing such transferable cues into a single prompt yields an over-smoothed consensus, which blurs semantically distinct patterns and weakens the balance between local specialization and global generalization.

A natural remedy is therefore to move from a single shared prompt to multi-expert prompts (MEPs), where multiple shared prompt experts explicitly capture different modes of transferable knowledge.
However, realizing an effective privacy-preserving FPT with MEPs immediately introduces two key challenges, as illustrated in Fig.~\ref{fig:motivation}.
\blackcircnum{1} \textbf{Privacy-Efficiency Dimensionality Curse}.
From the perspective of private optimization, introducing multiple shared experts substantially enlarges the dimensionality of communicated updates. 
Under local DP, this becomes particularly problematic because the expected squared norm of injected DP noise scales with the dimensionality of the privatized parameter space~\cite{zhou2021bypassing,pmlr-v235-wang24cu}. 
As a result, naively applying DP-SGD to high-dimensional MEPs can severely distort the delicate semantic geometry inherited from pretrained VLMs. Meanwhile, the linear growth of communication overhead with the number of expert prompts becomes prohibitive for bandwidth-limited edge devices.
\blackcircnum{2} \textbf{Robust Expert Composition Barrier}. 
From the perspective of expert utilization, the value of MEPs lies not in parameter expansion, but in whether multiple expert prompts can be seamlessly composed into an effective representation.
This becomes more pronounced under local DP, where each shared expert prompt is perturbed independently, making conventional routing strategies less effective. Sparse top-1 selection~\cite{fedus2022switch} can make the model brittle by over-relying on a single noisy expert prompt, whereas naive uniform averaging~\cite{huang2023experts} weakens instance-level adaptation by ignoring input differences. 
Therefore, achieving a better balance between robust cross-client transfer and instance-specific adaptation requires a more effective expert fusion mechanism that can mitigate the impact of any single noisy expert while preserving input-dependent expert specialization. 
These limitations lead to a pivotal question:
\textit{How can we equip privacy-preserving FPT with MEPs that achieve a stronger personalization--generalization trade-off, while remaining privacy-efficient under local DP?}

To this end, we propose \textbf{FedSEPT}, a novel privacy-preserving \textbf{Fed}erated \textbf{S}ubspace-decomposed \textbf{E}xpert \textbf{P}rompt \textbf{T}uning.
Motivated by prior evidence that adaptation updates of pretrained models often exhibit low-dimensional or low-rank structure~\cite{aghajanyan2021intrinsic,hu2022lora}, the key idea of FedSEPT is to enhance the expressiveness of shared knowledge while restricting private federated collaboration to a compact and stable space. 
\textbf{First}, to address \blackcircnum{1}, we introduce \textbf{S}ubspace-decomposed \textbf{E}xpert \textbf{M}odeling \textbf{(SEM)}, which represents each prompt expert using a shared low-rank factor, a fixed public basis, and a client-private residual. 
DP-SGD is applied exclusively to the compact shared factors, ensuring local DP while avoiding high-dimensional perturbation. 
Only these low-rank factors are communicated and aggregated, substantially reducing communication cost, while the fixed public basis provides a consistent parameterization across clients.
\textbf{Second}, to tackle \blackcircnum{2}, we develop \textbf{I}nstance-aware \textbf{E}xpert \textbf{F}usion \textbf{(IEF)}, where an on-device router combines semantically complementary experts into an input-dependent representation. 
This design improves robustness against noisy experts while retaining instance-wise adaptation, thereby circumventing the pitfalls of both hard top-1 selection and uniform averaging.
To make fusion efficient, we cache per-expert text features and perform fusion at the logit layer, avoiding repeated text encoding from dynamically mixed prompts.
Our main contributions are:
\begin{itemize}[leftmargin=1.5em, itemsep=0.2em, topsep=0.2em]
\item We propose \textbf{FedSEPT}, a federated MEP tuning framework that explicitly disentangles structured diversity under local DP.
\item We introduce \emph{Subspace-decomposed Expert Modeling} (SEM), leveraging low-rank synchronization and fixed public basis to ensure privacy-efficient sharing and consistent global aggregation.
\item  We design \emph{Instance-aware Expert Fusion} (IEF), enabling robust and efficient composition of semantically complementary noisy experts through on-device routing and logit-level fusion.
\item Extensive experiments on 11 benchmarks demonstrate that, under the same privacy constraints, FedSEPT achieves a superior trade-off between local adaptation and global generalization. 
Code is available at \url{https://github.com/yuCoryx/FedSEPT}.
\end{itemize}

\section{Related Work}
\label{sec:relatedwork}
\noindent \textbf{Personalized Federated Prompt Learning.\ }
Federated prompt learning (FPL) achieves parameter efficiency by collaboratively optimizing lightweight prompts over frozen backbones.
While early methods such as FedPrompt~\cite{zhao2023fedprompt} and PromptFL~\cite{guo2023promptfl} validated this approach, their reliance on a global prompt limits adaptability under heterogeneity.
To reconcile global sharing with local specialization, personalized FPL methods largely adopt decomposition-based designs. 
pFedPrompt~\cite{guo2023pfedprompt} optimizes local visual attention while maintaining shared language prompts. pFedPG~\cite{yang2023pFedPG} introduces a prompt generation framework, utilizing a server-side generator for personalized prompts. FedPGP~\cite{cui2024fedpgp} further models prompts as a superposition of shared and low-rank personalized components, FedOTP~\cite{li2024fedotp} introduces feature-level alignment via optimal transport, and FedPHA~\cite{fang2025fedpha} tackles structural heterogeneity via fixed-length global anchors and variable-length local carriers. 
Despite adapting effectively, these methods adhere to a \emph{single-shared-prompt} paradigm, which tends to over-smooth heterogeneous transferable patterns and restricts the modeling of diverse shared knowledge.  
A recent exception, pFedMoAP~\cite{luo2025pFedMoAP}, adopts a mixture-of-experts (MoE)~\cite{shazeer2017moe} composition. 
However, its proximity-based retrieval confines expert selection to similar peers, prioritizing experience reuse within local neighborhoods over discovering semantically complementary sources.  
As a result, it does not fully exploit the structural potential of multi-expert modeling for integrating diverse global knowledge toward cross-client generalization.

\vspace{1mm}
\noindent \textbf{Privacy-Preserving Federated Prompt Learning.\ }
Although FPL exchanges only a small number of prompt parameters, it still exposes client information through communicated updates and remains vulnerable to a semi-honest or curious server.  
Existing privacy-preserving FPL methods mainly fall into two categories.  
The first utilizes cryptographic protocols, such as secure aggregation or encrypted collaboration, to prevent direct access to client updates. For example, SecFPP~\cite{hou2025SecFPP} strengthens privacy protection through secure aggregation, but such methods typically incur substantial system overhead, which limits their practicality in cross-device settings. 
The second line adopts DP, where noise is injected into prompt-related updates to provide formal privacy guarantees. 
Representative efforts such as DP-FPL~\cite{tran2025dpfpl} offer better deployability, yet they struggle with the privacy-utility trade-off in high-dimensional prompt spaces and often overlook privacy risks during the training phase. 
More importantly, prior privacy-preserving FPL methods are largely tailored for single-prompt or weakly structured shared representations, and thus do not directly address how to preserve diverse transferable knowledge under local DP. 


\begin{figure*}[t]
    \centering
    \includegraphics[width=0.98\linewidth]{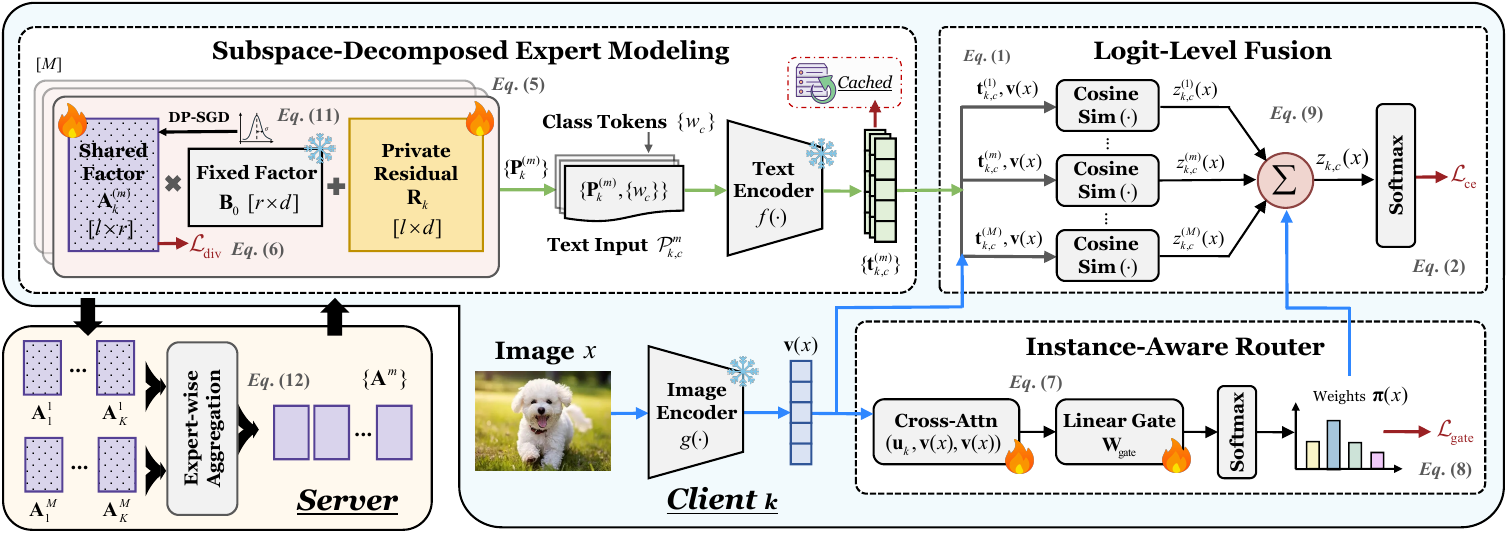}
    \vspace{-1mm}
    \caption{Overview of \textbf{FedSEPT}. 
    \ding{172} Each client constructs $M$ prompt experts as $\mathbf{P}_k^m=\mathbf{A}_k^m\mathbf{B}_0+\mathbf{R}_k$, where DP-SGD is applied only to the shared factors $\mathbf{A}_k^m$, the public basis $\mathbf{B}_0$ remains fixed, and the private residual $\mathbf{R}_k$ stays local.
    \ding{173} A local router predicts instance-aware expert weights from image features.
    \ding{174} Expert text features are cached and combined through logit-level fusion.
    \ding{175} Clients upload only privatized expert factors for server aggregation, while personalized parameters remain on-device.
    }
    \label{fig:overview}
\end{figure*}

\section{Preliminaries and Motivation}
\noindent \textbf{Federated Prompt Tuning (FPT).\ }
In an FL system with $K$ clients, each holding a private dataset $\mathcal{D}_k$, FPT adapts a frozen VLM—such as CLIP~\cite{radford2021clip}, comprising a text encoder $f(\cdot)$ and an image encoder $g(\cdot)$—to downstream tasks by optimizing only a learnable prompt $\mathbf{P}\in\mathbb{R}^{L\times d}$~\cite{guo2023promptfl}. Let $\mathcal{P}_c=\{\mathbf{P}, w_c\}$ be the prompted text input for class $c$, where $w_c$ is the class token. For an image $x$, the prediction probability for class $c$ is given by:
\begin{equation}
\label{eq:fpt_prob}
q(\hat{y}=c\mid x)=
\frac{\exp\big(\mathrm{sim}(g(x), f(\mathcal{P}_c))/\tau\big)}
{\sum_{j=1}^{C}\exp\big(\mathrm{sim}(g(x), f(\mathcal{P}_j))/\tau\big)},
\end{equation}
where $\mathrm{sim}(\cdot,\cdot)$ denotes cosine similarity and $\tau$ is the temperature.
Each client minimizes the local cross-entropy loss:
\begin{equation}
\label{eq:fpt_local_obj}
\min\ \mathbb{E}_{(x,y)\sim\mathcal{D}_k}\big[\mathcal{L}_{\mathrm{ce}}(x,y)\big],
\;
\mathcal{L}_{\mathrm{ce}}(x,y)=-\log q(\hat{y}=y\mid x).
\end{equation}
At each communication round $t$, clients upload their prompts $\mathbf{P}_k^{(t)}$, which the server aggregates via weighted averaging and broadcasts as the updated global prompt $\mathbf{P}^{(t+1)}$ for the next round.

\vspace{1mm}
\noindent \textbf{Differential Privacy (DP).\ }  
DP~\cite{dwork2006differential} ensures that a randomized mechanism's output is insensitive to any single training sample.

\begin{definition}[$(\epsilon,\delta)$-DP~\cite{dwork2006differential}]
A randomized mechanism $\mathcal{M}: \mathcal{X} \rightarrow \mathcal{R}$ satisfies $(\epsilon,\delta)$-DP if for any two neighboring datasets $\mathcal{D},\mathcal{D}'$ differing by one sample, and for any measurable subset ${S}\subseteq \mathcal{R}$,
\(\Pr[\mathcal{M}(\mathcal{D}) \in {S}]
\leq
e^{\epsilon}\Pr[\mathcal{M}(\mathcal{D}') \in {S}] + \delta,\)
where $\epsilon$ bounds the privacy loss and $\delta\in[0,1)$ allows a small failure probability.
\end{definition} 

In this work, we adopt the local DP setting to defend against an honest-but-curious or untrusted server. Unlike centralized DP, local DP requires each client to privatize its communicated updates before transmission~\cite{wei2021user,lee2025fedsvd}, so that the server never observes raw client gradients or unprotected prompt updates.

\vspace{1mm}
\noindent \textbf{Split-Prompt FPT under Local DP.\ } 
To cope with data heterogeneity while enforcing local DP, existing FPT methods frequently employ a personalized split-prompt architecture, which decomposes the prompt into a shared and a client-specific component~\cite{cui2024fedpgp,li2024fedotp,fang2025fedpha}:
\begin{equation}
\label{eq:split_prompt}
\mathbf{P}=\mathbf{P}_g+\mathbf{P}_l
\quad\text{or}\quad
\mathbf{P}=[\mathbf{P}_g:\mathbf{P}_l],
\end{equation}
where $\mathbf{P}_g$ captures transferable knowledge across clients, while $\mathbf{P}_l$ remains on-device for local adaptation. 
Under this paradigm, client-side DP-SGD~\cite{abadi2016deep,lee2025fedsvd} is applied strictly to the shared component $\mathbf{P}_g$.
Given a minibatch $\mathcal{B}$, the per-example gradients with respect to $\mathbf{P}_g$ are clipped and perturbed with Gaussian noise:
\begin{equation}
\label{eq:pfpt_dpsgd}
\tilde{\mathbf{g}}_{\mathcal{B}}
=
\frac{1}{|\mathcal{B}|}
\left[
\sum\nolimits_{(x_i,y_i)\in\mathcal{B}}
\textsc{Clip}\!\left(
\nabla_{\mathbf{P}_g}\ell(x_i,y_i),\,R
\right)
+
\mathcal{N}\!\left(
\mathbf{0},\,\sigma^2R^2\mathbf{I}
\right)
\right].
\end{equation}
where $\textsc{Clip}(\mathbf{g}_i, R) = \mathbf{g}_i \cdot \min(1, R/\|\mathbf{g}_i\|_2)$ is the clipping operator with threshold $R$, $\sigma$ is the noise multiplier. 
The shared prompt $\mathbf{P}_g$ is then updated using the privatized gradient before upload, while the client-specific component $\mathbf{P}_l$ remains local.

\noindent \textbf{Motivation.\ } 
Existing split-prompt FPT methods still rely on a single shared prompt $\mathbf{P}_g$ to capture transferable knowledge across clients. 
Under severe heterogeneity, diverse transferable cues from different clients are difficult to encode in one static shared context, which can lead to an over-homogenized shared prompt and a weaker personalization--generalization trade-off.

Transitioning to multi-expert prompts (MEPs) provides a logical pathway forward, theoretically allowing distinct experts to specialize in diverse transferable patterns. 
Yet, extending privacy-preserving FPT with MEPs is non-trivial. 
\textbf{First}, introducing multiple experts enlarges the communicated parameter space, so DP perturbation is injected into a higher-dimensional update and the communication burden increases accordingly.
\textbf{Second}, effective expert utilization becomes harder under local DP, since each expert is perturbed separately: sparse selection is brittle to noisy experts, whereas uniform averaging weakens input-specific adaptation. 
These challenges call for a federated MEP-tuning framework that is both privacy-efficient and capable of robust, input-adaptive expert fusion. 
Notations are summarized in Appendix~\ref{app:notations}.

\section{Methodology} 
\label{sec:method}
We present \textbf{FedSEPT}, a \textbf{Fed}erated \textbf{S}ubspace-decomposed \textbf{E}xpert \textbf{P}rompt \textbf{T}uning based on \emph{Subspace-decomposed Expert Modeling} (SEM) and \emph{Instance-aware Expert Fusion} (IEF), together with a federated private optimization procedure.

\subsection{Subspace-Decomposed Expert Modeling}
\label{subsec:method_main}

We begin by parameterizing multiple prompt experts in a compact shared space, while retaining client-specific structures locally.

\subsubsection{Expert Parameterization}
\label{sec:Expert Parameterization}

Each client maintains a library of $M$ prompt experts. For client $k$, the $m$-th expert prompt is parameterized as
\begin{equation}
\label{eq:expert_prompt}
\mathbf{P}_{k}^{m}
=
\mathbf{A}_{k}^{m}\mathbf{B}_0 + \mathbf{R}_k,
\qquad m=1,\ldots,M,
\end{equation}
where $\mathbf{A}_{k}^{m}\in\mathbb{R}^{L\times r}$ is the low-dimensional expert factor, $\mathbf{B}_0\in\mathbb{R}^{r\times d}$ is a public data-independent basis, and $\mathbf{R}_k\in\mathbb{R}^{L\times d}$ is a client-specific residual. Here, $L$ is the prompt length, $d$ is the embedding dimension, and $r\ll d$ is the subspace dimension. The component $\mathbf{A}_{k}^{m}\mathbf{B}_0$ captures transferable knowledge, whereas $\mathbf{R}_k$ models client-specific information and remains local. Therefore, only $\{\mathbf{A}_{k}^{m}\}_{m=1}^{M}$ are communicated and aggregated, separating global knowledge sharing from local personalization.

The basis $\mathbf{B}_0$ is initialized independently of client data, distributed to all clients, and fixed throughout training. This shared basis places expert factors in the same coordinate system, which is necessary because low-rank factorizations are not unique: for any orthogonal matrix $\mathbf{Q}\in\mathbb{R}^{r\times r}$,
\(\mathbf{A}\mathbf{B}
=
(\mathbf{A}\mathbf{Q}^{\top})(\mathbf{Q}\mathbf{B}).\) 
Learning client-specific bases would produce factor representations in incompatible coordinate systems, rendering direct averaging unreliable. By fixing a shared public basis, SEM aligns all expert factors within a common coordinate system and enables straightforward server aggregation. Meanwhile, the low-rank parameterization reduces the number of communicated parameters from $MLd$ to $MLr$.

\subsubsection{Expert Diversity Regularization}
\label{sec:expert_diversity}

Although the experts share the same basis, they should capture complementary patterns rather than redundant representations.
We stack the vectorized expert factors as $\mathbf{H}_k=[\operatorname{vec}(\mathbf{A}_k^1),\ldots,\operatorname{vec}(\mathbf{A}_k^M)]^\top\in\mathbb{R}^{M\times Lr}$, normalize its rows, and penalize the off-diagonal entries of the Gram matrix:
\begin{equation}
\label{eq:div}
\mathcal{L}_{\mathrm{div}}
=
\left\|
\operatorname{Norm}(\mathbf{H}_k)
\operatorname{Norm}(\mathbf{H}_k)^{\top}
\odot
\left(
\mathbf{1}_{M\times M}-\mathbf{I}_M
\right)
\right\|_F^2,
\end{equation}
where $\operatorname{Norm}(\cdot)$ denotes row-wise $\ell_2$ normalization, and $\mathbf{I}_M$ is the identity matrix. Minimizing $\mathcal{L}_{\mathrm{div}}$ discourages expert collapse and promotes specialization in complementary transferable knowledge.

\subsection{Instance-Aware Expert Fusion}
With a diverse set of experts established, 
we next design an input-dependent mechanism to compose them for each sample.

\subsubsection{Instance-Aware Routing}
\label{sec:Instance-Aware Routing}
We employ a lightweight client-specific router to predict expert weights from visual features. Given an image $x$ with normalized visual feature $\mathbf{v}(x)$, client $k$ extracts a task-specific representation using a learnable probe query $\mathbf{u}_k$ and computes the expert weights through a linear gate:
\begin{equation}
\label{eq:router}
\mathbf{h}_k
=
\operatorname{MultiHeadAttn}(\mathbf{u}_k,\mathbf{v}(x),\mathbf{v}(x)),
\,
\boldsymbol{\pi}_k(x)
=
\operatorname{Softmax}
(
\mathbf{W}_{\mathrm{gate},k}\mathbf{h}_k
).
\end{equation}
All routing parameters ($\mathbf{u}_k$ and $\mathbf{W}_{\mathrm{gate},k}$) are personalized on-device. 

To prevent a winner-takes-all collapse where a few dominant experts monopolize the assignments, we regularize the mean routing weights over a minibatch $\mathcal{B}$ using a load-balancing loss:
\begin{equation}
\label{eq:gate}
\bar{\boldsymbol{\pi}}_k
=
\frac{1}{|\mathcal{B}|}
\sum\nolimits_{x_i\in\mathcal{B}}
\boldsymbol{\pi}_k(x_i),
\,
\mathcal{L}_{\mathrm{gate}}
=
\sum\nolimits_{m=1}^{M}
\bar{\pi}_{k,m}
\log\left(M\bar{\pi}_{k,m}\right).
\end{equation}

\subsubsection{Logit-level Fusion} 
\label{sec:Logit-level Fusion}
Instance-aware composition could be performed at the prompt level by
$\mathbf{P}_k(x)=\big(\sum_{m=1}^{M}\pi_{k,m}(x)\mathbf{A}_k^m\big)\mathbf{B}_0+\mathbf{R}_k$.
However, such a design renders text features input-dependent, forcing the text encoder to re-evaluate every image-class pair and thereby negating the caching benefits of prompt tuning. We overcome this by shifting fusion to the logit level.

For each expert prompt $\mathbf{P}_{k}^{m}$ and class $c$, we compute the normalized text feature 
$\mathbf{t}_{k,c}^{m}=f(\mathcal{P}_{k,c}^{m})/\|f(\mathcal{P}_{k,c}^{m})\|_2,$
where $\mathcal{P}_{k,c}^{m}$ denotes the prompted class text using $\mathbf{P}_{k}^{m}$.
Crucially, these per-expert text features are computed once after each prompt update and cached locally for reuse across samples.
The per-expert class logits are obtained via standard cosine similarity, $z_{k,c}^{m}(x) = {\langle \mathbf{v}(x), \mathbf{t}_{k,c}^{m}\rangle}/{\tau}$. 
Finally, we fuse these independent expert predictions into a single logit using the predicted sample-wise routing weights:
\begin{equation}
\label{eq:logit_mix}
z_{k,c}(x)
{=}
\sum_{m=1}^{M}
\pi_{k,m}(x)z_{k,c}^{m}(x),
\,
p_k(\hat{y}{=}c\!\mid\! x)
=
\frac{\exp(z_{k,c}(x))}
{\sum_{j=1}^{C}\exp(z_{k,j}(x))}.
\end{equation} 
This late-fusion strategy replaces repeated text-encoder passes, which scale with $|\mathcal{B}|\cdot C$, by inexpensive vector dot-products and lightweight routing, making dynamic expert composition practical in both training and inference.

\subsection{Federated Private Optimization}
\label{subsec:private_optimization} 
During local training, client $k$ jointly optimizes the classification objective and the two structural regularizers:
\begin{equation}
\label{eq:total_loss}
\mathcal{L}_{\mathrm{total}}
=
\mathcal{L}_{\mathrm{ce}}
+
\lambda_{\mathrm{div}}\mathcal{L}_{\mathrm{div}}
+
\lambda_{\mathrm{gate}}\mathcal{L}_{\mathrm{gate}}.
\end{equation}

To enforce local DP, we apply DP-SGD only to the shared factors $\mathbf{A}_k=\{\mathbf{A}_k^m\}_{m=1}^{M}$. For a minibatch $\mathcal{B}$, they are updated as
\begin{equation}
\label{eq:subspace_dpsgd}
\mathbf{A}_k
=
\mathbf{A}_k
-
\frac{\eta}{|\mathcal{B}|}
\left[
\sum_{(x,y)\in\mathcal{B}}
\textsc{Clip}
\left(
\nabla_{\mathbf{A}_k}\ell(x,y),R
\right)
+
\mathcal{N}
\left(
0,\sigma_k^2R^2\mathbf{I}
\right)
\right].
\end{equation}
where $R$ is the clipping norm, $\sigma_k$ is the noise multiplier for client $k$.

After local optimization, each client $k$ uploads its privatized shared factors $\mathbf{A}_k$ to the server. 
Since all clients use the same fixed public basis, the server performs expert-wise weighted aggregation directly:
\begin{equation}
\label{eq:fedsep_agg}
\mathbf{A}^{m,(t+1)}
=
\sum_{k}
\frac{n_k}{\sum_{j} n_j}
\mathbf{A}_{k}^{m,(t)},
\quad m=1,\dots,M,
\end{equation}
where $\mathbf{A}^{m,(t+1)}$ denotes the updated global shared factor associated with the $m$-th expert at round $t+1$.
The aggregated set $\mathbf{A}^{(t+1)}=\{\mathbf{A}^{m,(t+1)}\}_{m=1}^{M}$ is then broadcast to participating clients for the next round.
The overall training procedure is summarized in Alg.~\ref{alg:fedsep}.

\vspace{1mm}
\noindent \textbf{Privacy Guarantee of FedSEPT.\ }
We now establish the privacy guarantee under the local DP threat model. 

\begin{theorem}[Privacy Guarantee]
\label{thm:privacy}
For client $k$, let $q_k=|\mathcal{B}|/n_k$ denote the minibatch sampling rate, and let $S_k$ be the total number of local DP-SGD steps performed over all communication rounds. For any $\delta\in(0,1)$, there exists a constant $c>0$ such that, if the noise multiplier satisfies \(\sigma_k \ge c \cdot
{q_k\sqrt{S_k\ln(1/\delta)}}/{\epsilon},\)
then the sequence of updates released by client $k$ satisfies $(\epsilon,\delta)$-DP.
\end{theorem}

\begin{proof}
At each local step, FedSEPT clips per-sample gradients of the shared expert factors and adds Gaussian noise before updating the released parameters. Each step is thus a subsampled Gaussian mechanism with sampling rate $q_k$ and noise multiplier $\sigma_k$. By the moments-accountant analysis of DP-SGD~\cite{abadi2016deep}, composing privacy loss over $S_k$ steps yields the stated $(\epsilon,\delta)$-DP guarantee. Subsequent server aggregation based only on privatized factors incurs no additional privacy loss by DP post-processing~\cite{dwork2006differential}.
\end{proof}

\vspace{-2mm}

\begin{algorithm}[t]
\caption{FedSEPT}
\label{alg:fedsep}
\small
\begin{algorithmic}[1]
\REQUIRE Communication rounds $T$, local epochs $E$, number of experts $M$
\STATE Initialize the global expert factors $\mathbf{A}^{(0)}$ and distribute the public basis $\mathbf{B}_0$
\STATE Initialize the local private parameters
$\boldsymbol{\theta}_k^{\mathrm{pri}}
=\{\mathbf{R}_k,\mathbf{u}_k,\mathbf{W}_{\mathrm{gate},k}\}$
\FOR{$t=0,\ldots,T-1$}
    \STATE Broadcast $\mathbf{A}^{(t)}$ to all clients
    \FOR{each client $k=1,\ldots,K$ \textbf{in parallel}}
        \STATE Set $\mathbf{A}_k \leftarrow \mathbf{A}^{(t)}$
        \FOR{local epoch $e=1,\ldots,E$}
            \FOR{each minibatch $\mathcal{B}\subset\mathcal{D}_k$}
                \STATE Construct the expert prompts using Eq.~\eqref{eq:expert_prompt}
                \STATE Compute the instance-aware routing weights using Eq.~\eqref{eq:router}
                \STATE Compute the fused predictions using Eq.~\eqref{eq:logit_mix}
                \STATE Compute the total loss $\mathcal{L}_{\mathrm{total}}$ using Eq.~\eqref{eq:total_loss}
                \STATE Update $\boldsymbol{\theta}_k^{\mathrm{pri}}$ locally and update $\mathbf{A}_k$ using DP-SGD by Eq.~\eqref{eq:subspace_dpsgd}
            \ENDFOR
        \ENDFOR
        \STATE Upload the privatized expert factors $\mathbf{A}_k$
    \ENDFOR
    \STATE Aggregate $\{\mathbf{A}_k\}_{k=1}^{K}$ using Eq.~\eqref{eq:fedsep_agg} to obtain $\mathbf{A}^{(t+1)}$
\ENDFOR
\STATE \textbf{Output:} $\mathbf{A}^{(T)}$ and
$\{\boldsymbol{\theta}_k^{\mathrm{pri}}\}_{k=1}^{K}$
\end{algorithmic}
\end{algorithm}

\section{Experiments}\label{sec:Experiments} 
 
\begin{table*}[t]
\centering
\caption{Average test accuracy (\%) under pathological label skew at $\epsilon=1$. The \textbf{best} and the \underline{second} are marked.}
\vspace{-2mm}
\label{tab:pathological_label_skew}
\fontsize{8}{8}\selectfont
\renewcommand{\arraystretch}{1.15}
\setlength{\tabcolsep}{0.8pt}

\begin{tabularx}{\linewidth}{l||*{3}{>{\centering\arraybackslash}X}|*{3}{>{\centering\arraybackslash}X}|*{3}{>{\centering\arraybackslash}X}|*{3}{>{\centering\arraybackslash}X}|*{3}{>{\centering\arraybackslash}X}}
\noalign{\hrule height 0.8pt}
\rowcolor{gray!25}
& \multicolumn{3}{c|}{\textbf{Food101}}
& \multicolumn{3}{c|}{\textbf{Caltech101}}
& \multicolumn{3}{c|}{\textbf{Flowers102}}
& \multicolumn{3}{c|}{\textbf{DTD}}
& \multicolumn{3}{c}{\textbf{OxfordPets}} \\
\cline{2-4}\cline{5-7}\cline{8-10}\cline{11-13}\cline{14-16}

\rowcolor{gray!25}
\multicolumn{1}{c||}{\multirow{-2}{*}{Methods}}
& In & Cross & HM
& In & Cross & HM
& In & Cross & HM
& In & Cross & HM
& In & Cross & HM \\
\hline\hline

PromptFL~\cite{guo2023promptfl}
& 73.81 & \underline{73.86} & 73.83
& 81.91 & 81.42 & 81.66
& 16.74 & 17.39 & 17.06
& 16.28 & 17.31 & 16.78
& 67.97 & 65.52 & 66.72 \\

\rowcolor{gray!10}
FedPGP~\cite{cui2024fedpgp}
& 66.80 & 62.36 & 64.50
& 84.01 & 82.17 & 83.08
& 26.66 & 25.07 & 25.84
& 28.70 & 13.48 & 18.34
& 65.10 & 61.64 & 63.32 \\

FedOTP~\cite{li2024fedotp}
& \textbf{95.56} & 45.35 & 61.51
& \textbf{98.98} & 67.32 & 80.14
& 86.12 & 42.13 & 56.58
& \underline{91.42} & 11.64 & 20.65
& \textbf{99.10} & 50.41 & 66.83 \\

\rowcolor{gray!10}
FedPHA~\cite{fang2025fedpha}
& \underline{95.34} & 25.48 & 40.21
& \underline{98.96} & 55.52 & 71.13
& \textbf{95.79} & 22.38 & 36.28
& \textbf{92.27} & 7.28 & 13.50
& \underline{98.98} & 19.60 & 32.72 \\

pFedMoAP~\cite{luo2025pFedMoAP}
& 81.04 & 69.40 & 74.77
& 83.41 & 81.38 & 82.38
& 84.00 & \underline{49.25} & \underline{62.09}
& 18.32 & 19.17 & 18.73
& 74.10 & 67.01 & \underline{70.38} \\


\rowcolor{gray!10}
DP-FPL~\cite{tran2025dpfpl}
& 81.46 & \textbf{79.89} & \underline{80.67}
& 88.02 & \textbf{87.28} & \underline{87.65}
& 80.70 & 44.66 & 57.50
& 22.92 & \underline{22.71} & \underline{22.82}
& 70.43 & \underline{66.28} & 68.29 \\

\hline
\rowcolor{lightblue}
FedSEPT
& 92.22 & 71.82 & \textbf{80.75}
& 97.68 & \underline{82.40} & \textbf{89.39}
& \underline{92.72} & \textbf{50.11} & \textbf{65.06}
& 79.89 & \textbf{25.99} & \textbf{39.22}
& 97.56 & \textbf{83.95} & \textbf{90.24} \\

\noalign{\hrule height 0.8pt}
\end{tabularx}
\end{table*} 

\begin{table*}[t]
\centering
\caption{Average test accuracy (\%) under practical label skew at $\epsilon=1$. The \textbf{best} and the \underline{second} are marked.}
\vspace{-2mm}
\label{tab:practical_label_skew}
\fontsize{8}{8}\selectfont
\renewcommand{\arraystretch}{1.15}
\setlength{\tabcolsep}{0.8pt}
\begin{tabularx}{\linewidth}{l||*{3}{>{\centering\arraybackslash}X}|*{3}{>{\centering\arraybackslash}X}|*{3}{>{\centering\arraybackslash}X}|*{3}{>{\centering\arraybackslash}X}|*{3}{>{\centering\arraybackslash}X}|*{3}{>{\centering\arraybackslash}X}}
\noalign{\hrule height 0.8pt}
\rowcolor{gray!25}
& \multicolumn{9}{c|}{\textbf{CIFAR-10}} & \multicolumn{9}{c}{\textbf{CIFAR-100}} \\
\cline{2-10}\cline{11-19}
\rowcolor{gray!25}
& \multicolumn{3}{c}{$\beta=0.1$} & \multicolumn{3}{c}{$\beta=0.3$} & \multicolumn{3}{c|}{$\beta=0.5$} & \multicolumn{3}{c}{$\beta=0.1$} & \multicolumn{3}{c}{$\beta=0.3$} & \multicolumn{3}{c}{$\beta=0.5$} \\
\rowcolor{gray!25}
\multicolumn{1}{c||}{\multirow{-3}{*}{Methods}} & In & Cross & HM & In & Cross & HM & In & Cross & HM & In & Cross & HM & In & Cross & HM & In & Cross & HM \\
\hline\hline
PromptFL~\cite{guo2023promptfl}
& 86.03 & \underline{85.50} & 85.76 & 87.32 & 85.77 & 86.54 & 87.03 & 85.65 & 86.33 & 61.43 & \textbf{61.47} & 61.45 & 62.04 & \underline{61.98} & 62.01 & 61.87 & 61.86 & 61.86 \\
\rowcolor{gray!10}
FedPGP~\cite{cui2024fedpgp}
& 84.63 & 83.62 & 84.12 & 87.25 & \underline{86.40} & \underline{86.82} & 85.68 & 85.02 & 85.35 & 54.89 & 54.99 & 54.94 & 51.17 & 51.36 & 51.26 & 56.05 & 55.78 & 55.91 \\
FedOTP~\cite{li2024fedotp}
& \underline{96.63} & 59.97 & 74.01 & \textbf{95.54} & 76.86 & 85.19 & \underline{93.77} & 82.14 & \underline{87.57} & \underline{84.61} & 48.27 & \underline{61.47} & \underline{77.20} & 59.29 & \underline{67.07} & \underline{73.91} & \underline{62.90} & \underline{67.96} \\
\rowcolor{gray!10}
FedPHA~\cite{fang2025fedpha}
& \textbf{96.69} & 54.95 & 70.08 & \underline{95.28} & 70.90 & 81.30 & \textbf{93.79} & 79.29 & 85.93 & \textbf{85.64} & 46.92 & 60.63 & \textbf{78.28} & 56.62 & 65.71 & \textbf{74.38} & 60.49 & 66.72 \\
pFedMoAP~\cite{luo2025pFedMoAP}
& 91.89 & 79.43 & 85.21
& 91.93 & 81.92 & 86.64
& 90.84 & 83.28 & 86.90
& 63.64 & 57.07 & 60.17 
& 61.11 & 58.71 & 59.89 
& 60.54 & 58.77 & 59.64 \\ 
\rowcolor{gray!10}
DP-FPL~\cite{tran2025dpfpl}
& 86.42 & \textbf{86.49} & \underline{86.45} & 87.01 & \textbf{86.63} & \underline{86.82} & 87.37 & \textbf{87.07} & 87.22 & 56.45 & 54.65 & 55.54 & 54.29 & 54.15 & 54.22 & 56.26 & 56.32 & 56.29 \\
\hline
\rowcolor{lightblue}
FedSEPT
& 92.59 & 83.48 & \textbf{87.80}
& 92.77 & 85.81 & \textbf{89.15} 
& 90.85 & \underline{86.86} & \textbf{88.81} 
& 77.54 & \underline{58.67} & \textbf{66.80} 
& 72.42 & \textbf{63.64} & \textbf{67.75} 
& 70.27 & \textbf{65.96} & \textbf{68.05} \\
\noalign{\hrule height 0.8pt}
\end{tabularx}
\end{table*} 

\subsection{Experimental Setup} 
\noindent \textbf{Datasets.\ } 
We evaluate FedSEPT across \textbf{11} benchmark datasets, categorized into three statistical heterogeneity settings.
\begin{itemize}[leftmargin=1.5em, itemsep=0.2em, topsep=0.2em]
\item \textbf{Pathological label skew}. For fine-grained classification, we use five CLIP benchmark datasets: Food101~\cite{bossard2014food}, Caltech101~\cite{10.5555/1032643.1033069}, Flowers102~\cite{flower4756141}, DTD~\cite{10.1109/CVPR.2014.461} and OxfordPets~\cite{parkhi12a}. 
Following standard few-shot protocols~\cite{cui2024fedpgp,fang2025fedpha}, each client is assigned disjoint class sets to construct strict label heterogeneity.

\item \textbf{Practical label skew}. 
For larger-scale and more realistic non-IID settings, we partition CIFAR-10 and CIFAR-100~\cite{krizhevsky2009learning} via a symmetric Dirichlet distribution $\beta \in \{0.1, 0.3, 0.5\}$, thereby simulating heterogeneous client-level label distributions in edge networks.

\item \textbf{Domain and label skews}.
To assess robustness against compound shifts, we employ four domain adaptation benchmarks: PACS (4 domains)~\cite{gong2012geodesic}, Office31 (3 domains)~\cite{saenko2010adapting}, OfficeHome (4 domains)~\cite{venkateswara2017deep}, and DomainNet (5 domains)~\cite{peng2019moment}.
We apply a hierarchical partitioning strategy by first allocating specific domains to distinct clients, and then applying a secondary Dirichlet partition $\beta \in \{0.1, 0.3, 0.5\}$ to inject intra-domain label skew.
\end{itemize}
This rigorously tests the model's ability to disentangle domain styles from class features.
Details are provided in Appendix~\ref{app:dataset}.

\vspace{1mm}
\noindent \textbf{Counterparts.\ }
We compare FedSEPT with personalized federated prompt learning methods, including PromptFL~\cite{guo2023promptfl}, FedPGP~\cite{cui2024fedpgp}, FedOTP~\cite{li2024fedotp}, FedPHA~\cite{fang2025fedpha}, pFedMoAP~\cite{luo2025pFedMoAP}, and the privacy-preserving baseline DP-FPL~\cite{tran2025dpfpl}. For a fair comparison under the same local-DP threat model, we apply DP only to the parameters uploaded to the server under an identical privacy budget. Accordingly, for DP-FPL, noise is added only to its communicated prompt component, while the locally decomposed prompt component remains unperturbed and is never uploaded. The non-private counterparts are similarly equipped with local DP on their shared global prompts.

\vspace{1mm}
\noindent \textbf{Implementation Details.\ }
We follow a unified implementation protocol~\cite{fang2025fedpha, yang2025feddda}. 
We use the pre-trained CLIP~\cite{radford2021clip} with a ViT-B/16 backbone. 
We run communication rounds $T = 50$ with local epochs $E=1$. Training and testing batch sizes are 32 and 128, respectively. Local optimization is SGD with learning rate 0.001. 
For the CLIP datasets, we set $K=10$ clients, each holding a disjoint subset of classes. 
For CIFAR-10 and CIFAR-100, we scale to $K=100$ clients with a 10\% per-round participation rate to mimic realistic device availability. For multi-domain datasets, data from each domain is split across two clients, making $K$ twice the number of domains.
Prompts of length $L=16$ and dimension $d=512$ are initialized from a normal distribution. 
Unless otherwise specified, we set the privacy budget to $\epsilon=1.0$ and $\delta=10^{-5}$, with clipping threshold $R=1.0$. The noise scale $\sigma$ is computed according to Theorem~\ref{thm:privacy}. 
FedSEPT uses $M=4$ experts, a subspace rank of $r=16$, and four attention heads for routing, with $\lambda_{\mathrm{div}}=\lambda_{\mathrm{gate}}=10$.
More details are provided in Appendix~\ref{app:baseline_hyperparameters}.

\vspace{1mm}
\noindent \textbf{Evaluation Metrics.\ }
We report two complementary metrics under heterogeneous federated settings.
\textit{In-Client} denotes the accuracy on each client’s test set, reflecting client-level personalization ability.
\textit{Cross-Client} denotes the accuracy on the aggregated test sets of all other clients, reflecting generalization beyond the local client distribution.
We further report their \textit{harmonic mean (HM)},
$\mathrm{HM} = \frac{2 \cdot \mathrm{Acc}_{\text{in}} \cdot \mathrm{Acc}_{\text{cross}}}{\mathrm{Acc}_{\text{in}} + \mathrm{Acc}_{\text{cross}}}$,
which better captures the trade-off between personalization and cross-client generalization.
For each method, we report the mean client performance over the last 5 rounds.

\subsection{Performance Comparison} 

\begin{figure}[t]
    \centering
   \includegraphics[width=1\linewidth]{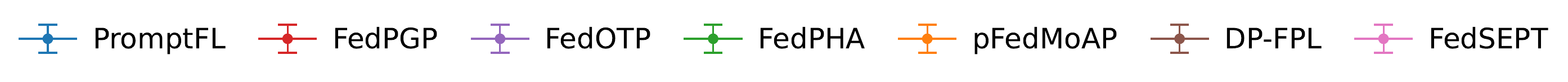}
   \vspace{1mm}
   \includegraphics[width=0.98\linewidth]{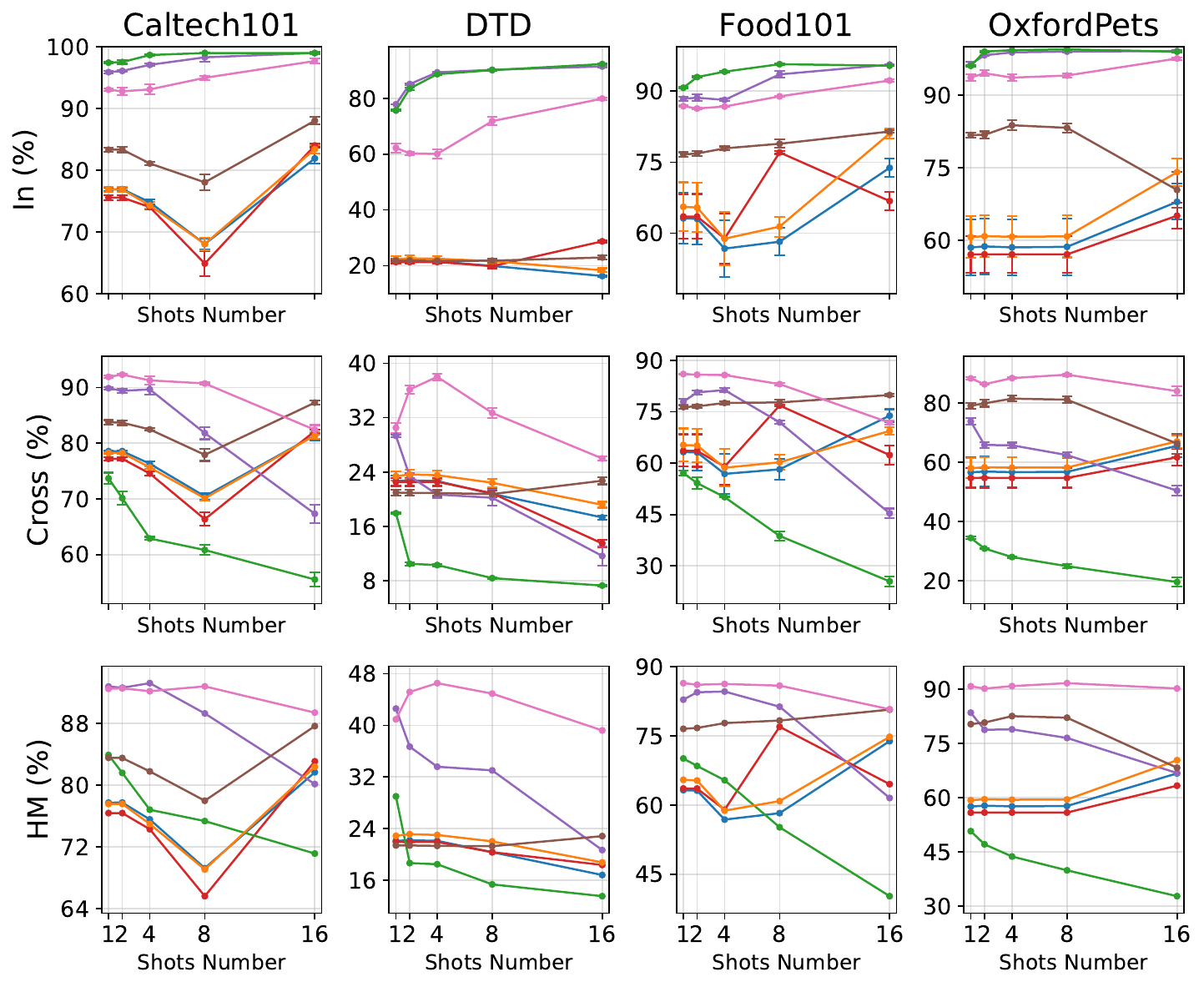}
   \vspace{-2mm}
   \caption{Effect of the number of shots.}
    \label{fig:shots_number_grid}
\end{figure}

\begin{table*}[t]
\centering
\caption{Average test accuracy (\%) under domain and label skews at $\epsilon=1$. The \textbf{best} and the \underline{second} are marked.}
\vspace{-2mm}
\label{tab:feature_label_skews}
\fontsize{8}{8}\selectfont
\renewcommand{\arraystretch}{1.15}
\setlength{\tabcolsep}{0.8pt}

\begin{tabularx}{\linewidth}{l||*{3}{>{\centering\arraybackslash}X}|*{3}{>{\centering\arraybackslash}X}|*{3}{>{\centering\arraybackslash}X}|*{3}{>{\centering\arraybackslash}X}|*{3}{>{\centering\arraybackslash}X}|*{3}{>{\centering\arraybackslash}X}}
\noalign{\hrule height 0.8pt}
\rowcolor{gray!25}
& \multicolumn{9}{c|}{\textbf{PACS}} & \multicolumn{9}{c}{\textbf{Office31}} \\
\cline{2-10}\cline{11-19}
\rowcolor{gray!25}
& \multicolumn{3}{c}{$\beta=0.1$}
& \multicolumn{3}{c}{$\beta=0.3$}
& \multicolumn{3}{c|}{$\beta=0.5$}
& \multicolumn{3}{c}{$\beta=0.1$}
& \multicolumn{3}{c}{$\beta=0.3$}
& \multicolumn{3}{c}{$\beta=0.5$} \\
\rowcolor{gray!25}
\multicolumn{1}{c||}{\multirow{-3}{*}{Methods}}
& In & Cross & HM
& In & Cross & HM
& In & Cross & HM
& In & Cross & HM
& In & Cross & HM
& In & Cross & HM \\
\hline\hline

PromptFL~\cite{guo2023promptfl}
& 97.15 & 92.16 & 94.59
& 96.13 & 94.97 & 95.55
& 95.49 & 94.50 & 94.99
& 78.00 & 79.54 & 78.76
& 78.54 & 79.59 & 79.06
& 77.29 & 78.59 & 77.93 \\

\rowcolor{gray!10}
FedPGP~\cite{cui2024fedpgp}
& 96.51 & 83.36 & 89.45
& 90.15 & 88.57 & 89.35
& 96.96 & 92.63 & 94.75
& 64.49 & 60.31 & 62.33
& 63.18 & 62.41 & 62.79
& 70.23 & 69.49 & 69.86 \\

FedOTP~\cite{li2024fedotp}
& \textbf{98.70} & 87.12 & 92.55
& \underline{98.03} & 92.73 & 95.31
& \underline{98.07} & 93.28 & 95.62
& \underline{94.73} & 78.44 & \underline{85.82}
& \underline{93.40} & \underline{82.38} & \underline{87.54}
& \underline{93.76} & \underline{83.53} & 88.35 \\

\rowcolor{gray!10}
FedPHA~\cite{fang2025fedpha}
& \underline{98.63} & 85.41 & 91.55
& \textbf{98.15} & 91.81 & 94.87
& \textbf{98.15} & 91.92 & 94.93
& \textbf{95.75} & 76.93 & 85.31
& \textbf{94.02} & 81.40 & 87.26
& \textbf{94.88} & 82.68 & \underline{88.36} \\

pFedMoAP~\cite{luo2025pFedMoAP}
& 97.97 & 87.72 & 92.56
& 97.80 & 92.57 & 95.11
& 97.67 & 92.75 & 95.15
& 83.64 & 77.61 & 80.52
& 83.45 & 79.87 & 81.62
& 83.49 & 79.70 & 81.55 \\

\rowcolor{gray!10}
DP-FPL~\cite{tran2025dpfpl}
& 97.15 & \underline{92.26} & \underline{94.64}
& 97.37 & \underline{95.08} & \underline{96.21}
& 97.39 & \underline{95.30} & \underline{96.33}
& 79.72 & \underline{80.48} & 80.10
& 79.71 & 80.62 & 80.16
& 81.09 & 81.72 & 81.41 \\

\hline
\rowcolor{lightblue}
FedSEPT
& 97.61 & \textbf{92.54} & \textbf{95.01}
& 97.94 & \textbf{95.21} & \textbf{96.56}
& 98.06 & \textbf{95.65} & \textbf{96.84}
& 91.48 & \textbf{81.49} & \textbf{86.20}
& 91.70 & \textbf{84.45} & \textbf{87.93}
& 92.87 & \textbf{84.50} & \textbf{88.49} \\

\noalign{\hrule height 0.8pt}
\end{tabularx}

\vspace{1mm}

\begin{tabularx}{\linewidth}{l||*{3}{>{\centering\arraybackslash}X}|*{3}{>{\centering\arraybackslash}X}|*{3}{>{\centering\arraybackslash}X}|*{3}{>{\centering\arraybackslash}X}|*{3}{>{\centering\arraybackslash}X}|*{3}{>{\centering\arraybackslash}X}}
\noalign{\hrule height 0.8pt}
\rowcolor{gray!25}
& \multicolumn{9}{c|}{\textbf{OfficeHome}}
& \multicolumn{9}{c}{\textbf{DomainNet}} \\
\cline{2-10}\cline{11-19}
\rowcolor{gray!25}
& \multicolumn{3}{c}{$\beta=0.1$}
& \multicolumn{3}{c}{$\beta=0.3$}
& \multicolumn{3}{c|}{$\beta=0.5$}
& \multicolumn{3}{c}{$\beta=0.1$}
& \multicolumn{3}{c}{$\beta=0.3$}
& \multicolumn{3}{c}{$\beta=0.5$} \\
\rowcolor{gray!25}
\multicolumn{1}{c||}{\multirow{-3}{*}{Methods}}
& In & Cross & HM
& In & Cross & HM
& In & Cross & HM
& In & Cross & HM
& In & Cross & HM
& In & Cross & HM \\
\hline\hline

PromptFL~\cite{guo2023promptfl}
& 79.98 & \underline{79.89} & 79.93
& 79.71 & 79.66 & 79.68
& 80.05 & 79.93 & 79.99
& 71.88 & \underline{61.72} & 66.41
& 71.90 & \underline{61.34} & 66.20
& 71.77 & 62.05 & 66.56 \\

\rowcolor{gray!10}
FedPGP~\cite{cui2024fedpgp}
& 79.92 & 76.25 & 78.04
& 77.86 & 75.35 & 76.58
& 79.02 & 77.41 & 78.21
& 74.61 & 57.09 & 64.68
& 73.01 & 60.60 & 66.23
& 73.26 & 61.50 & 66.87 \\

FedOTP~\cite{li2024fedotp}
& \textbf{89.92} & 73.00 & 80.58
& \underline{88.68} & 78.95 & \underline{83.53}
& \underline{88.35} & 79.53 & \underline{83.71}
& 82.60 & 54.19 & 65.45
& \textbf{82.09} & 56.01 & 66.59
& \textbf{81.92} & 57.79 & 67.77 \\

\rowcolor{gray!10}
FedPHA~\cite{fang2025fedpha}
& \underline{89.47} & 73.39 & 80.64
& \textbf{88.90} & 77.61 & 82.87
& \textbf{88.38} & 78.60 & 83.20
& \textbf{83.07} & 51.70 & 63.73
& \underline{81.91} & 56.80 & 67.08
& \underline{81.61} & 57.58 & 67.52 \\

pFedMoAP~\cite{luo2025pFedMoAP}
& 84.70 & 76.24 & 80.25
& 83.02 & 78.82 & 80.86
& 83.17 & 79.16 & 81.11
& 70.91 & \textbf{62.08} & 66.20
& 80.37 & 57.74 & 67.20
& 80.16 & 58.09 & 67.36 \\

\rowcolor{gray!10}
DP-FPL~\cite{tran2025dpfpl}
& 82.40 & \textbf{81.60} & \underline{82.00}
& 82.00 & \textbf{81.26} & 81.63
& 82.39 & \underline{81.88} & 82.14
& 76.59 & 59.22 & \underline{66.79}
& 75.34 & \textbf{61.38} & \underline{67.65}
& 75.18 & \textbf{63.45} & \underline{68.82} \\

\hline
\rowcolor{lightblue}
FedSEPT
& 88.81 & 78.69 & \textbf{83.44}
& 87.89 & \underline{80.90} & \textbf{84.25}
& 87.48 & \textbf{82.22} & \textbf{84.77}
& \underline{83.03} & 56.25 & \textbf{67.07}
& 81.49 & 59.17 & \textbf{68.56}
& 81.26 & \underline{61.31} & \textbf{69.89} \\

\noalign{\hrule height 0.8pt}
\end{tabularx}
\end{table*}

\begin{figure*}[t]
\centering
\setlength{\tabcolsep}{3pt}
\begin{tabular}{@{}ccc@{}}
\includegraphics[
    width=0.195\textwidth
]{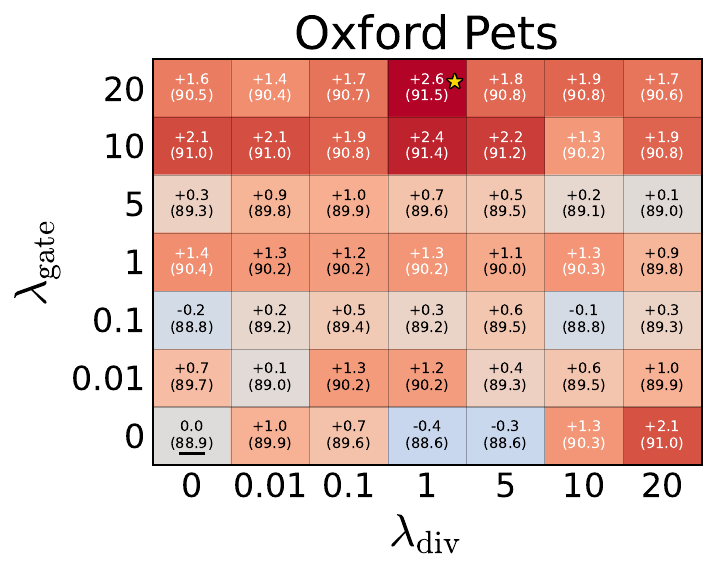}
&
\includegraphics[
    width=0.385\textwidth
]{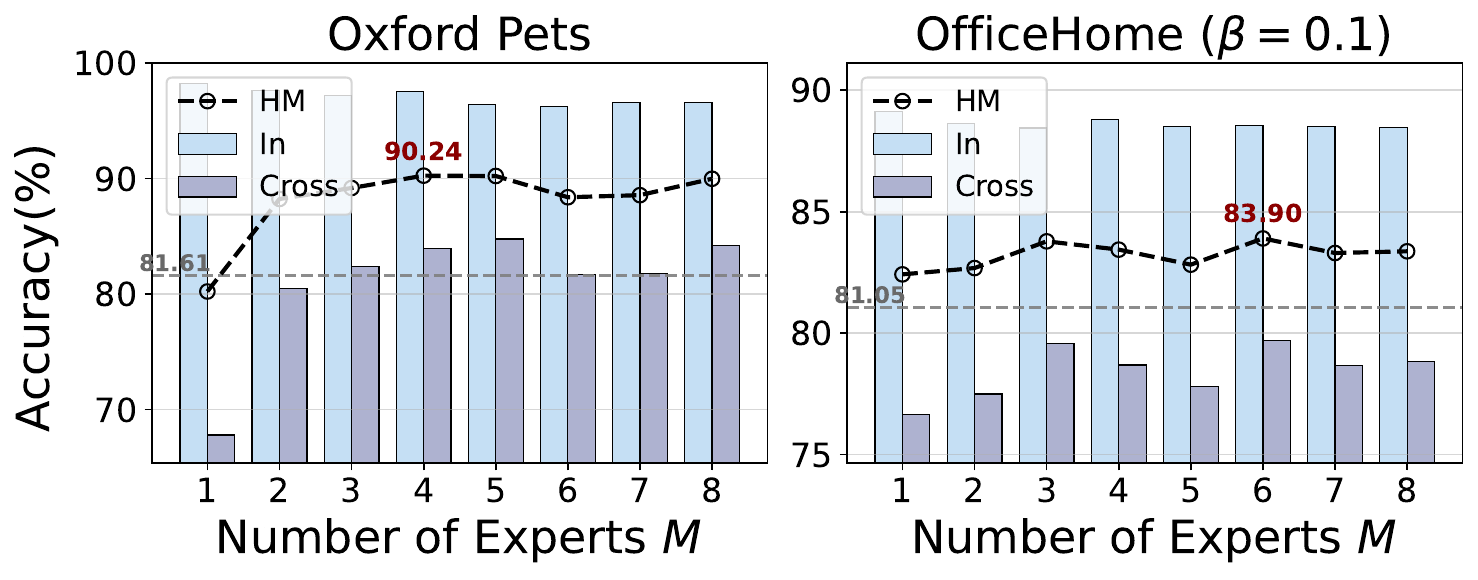}
&
\includegraphics[
    width=0.385\textwidth
]{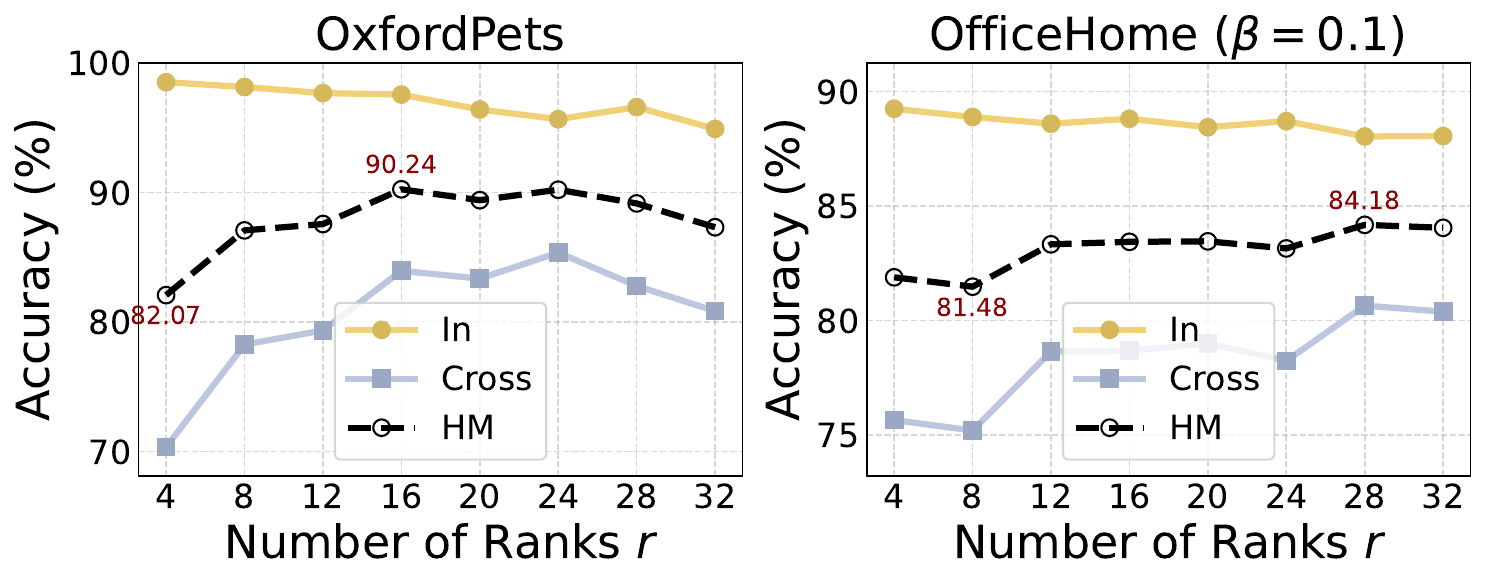}
\\[-1mm]
\small \textbf{(a) Weight sensitivity} &
\small \textbf{(b) Effect of the low-rank expert} &
\small \textbf{(c) Effect of the subspace rank}
\end{tabular}

\vspace{-2mm}
\caption{Hyperparameter analysis.}
\label{fig:hyperparams}
\end{figure*}

\noindent \textbf{Pathological Label Skew and Data Scarcity.\ }
As shown in Tab.~\ref{tab:pathological_label_skew}, FedSEPT achieves the best \textit{HM} on all five datasets under pathological label skew. Its gains are particularly pronounced on Food101, DTD, and OxfordPets, where stronger \textit{Cross-Client} accuracy than highly personalized methods such as FedOTP and FedPHA leads to substantially higher \textit{HM} while retaining competitive
\textit{In-Client} performance. Fig.~\ref{fig:shots_number_grid} further shows that reducing the number of shots mainly impairs cross-client generalization. Nevertheless, FedSEPT maintains strong \textit{Cross-Client} accuracy and competitive \textit{In-Client} performance across most low-shot settings. These results indicate that FedSEPT mitigates over-personalization and remains robust under both severe label isolation and limited local supervision.

\vspace{1mm}
\noindent \textbf{Practical Label Skew.\ }
Tab.~\ref{tab:practical_label_skew} shows that FedSEPT remains consistently effective in larger-scale and more realistic federated settings. 
While PromptFL and FedPGP are competitive in some \textit{Cross-Client} cases, they are generally weaker in \textit{In-Client} accuracy; conversely, FedOTP and FedPHA often preserve stronger local accuracy but generalize less effectively across clients.
As a result, FedSEPT achieves the best overall trade-off, especially on CIFAR-100. 

\vspace{1mm}
\noindent \textbf{Domain and Label Skews.\ }
From Tab.~\ref{tab:feature_label_skews}, FedSEPT achieves the best HM across all domain-and-label-skew settings. On PACS, the \textit{HM} gains primarily stem from stronger \textit{Cross-Client} accuracy while retaining competitive \textit{In-Client} performance. On Office31 and OfficeHome, the gains are driven primarily by substantially stronger \textit{In-Client} accuracy while maintaining competitive \textit{Cross-Client} performance. On DomainNet, FedSEPT achieves a better balance between local fitting and cross-client transfer.

\subsection{Diagnostic Analysis}

\noindent \textbf{Hyperparameter Analysis.\ }
As shown in Fig.~\ref{fig:hyperparams}, most combinations of $\mathcal{L}_{\mathrm{div}}$ and $\mathcal{L}_{\mathrm{gate}}$ outperform the unregularized baseline, with stronger gate regularization generally yielding stable improvements and a broad high-performing region. Increasing the number of experts initially enhances transferable diversity, whereas larger $M$ provides no consistent gains and may introduce redundancy. Similarly, a moderate subspace rank offers a favorable balance between representation capacity and parameter efficiency, while larger ranks yield only limited improvements. We therefore adopt $\lambda_{\mathrm{div}}=\lambda_{\mathrm{gate}}=10$, $M=4$, and $r=16$ as robust default settings.

\vspace{1mm}
\noindent \textbf{Ablation Study.\ }
We evaluate the main components under three representative data skews.
\textbf{Single} removes both SEM and IEF, reducing FedSEPT to a single full-dimensional prompt.
\textbf{Full MEPs} retains multiple full-dimensional experts but removes the low-rank decomposition, while \textbf{w/o $\mathbf{R}_k$} removes the private residual.
For fusion, \textbf{Top-1} performs hard expert selection, whereas \textbf{Uniform} assigns equal weights without instance-aware routing.
All component variants yield lower HM than FedSEPT across OxfordPets, CIFAR-10, and OfficeHome.
Full MEPs remains consistently inferior, showing that the gains do not arise merely from adding more prompts.
Removing $\mathbf{R}_k$ causes substantial degradation, especially on OxfordPets and OfficeHome, confirming the value of local personalization.
Top-1 and Uniform also underperform the full model, demonstrating the effectiveness of soft instance-aware fusion.

\begin{table}[t]
\centering
\caption{Ablation study of FedSEPT under three skews.}
\vspace{-2mm}
\label{tab:ablation_multi_full_selected_hm_partial}
\fontsize{8}{8}\selectfont
\renewcommand{\arraystretch}{1.2}
\setlength{\tabcolsep}{2pt}
\begin{tabular}{l|ccc|ccc|ccc}
\noalign{\hrule height 0.8pt}
\rowcolor{gray!25}
& \multicolumn{3}{c|}{OxfordPets}
& \multicolumn{3}{c|}{CIFAR-10 ($\beta=0.1$)}
& \multicolumn{3}{c}{OfficeHome ($\beta=0.1$)} \\
\cline{2-10}
\rowcolor{gray!25}
\multirow{-2}{*}{Method}
& In & Cross & HM
& In & Cross & HM
& In & Cross & HM \\
\hline\hline

Single
& 78.05 & 75.22 & 76.61
& 85.02 & 85.33 & 85.17
& 81.07 & 80.07 & 80.57 \\

\rowcolor{gray!10}
Full MEPs
& 77.66 & 73.92 & 75.74
& 86.20 & 86.62 & 86.41
& 81.62 & 80.85 & 81.24 \\

w/o $\mathbf{R}_k$
& 74.89 & 72.08 & 73.46
& 86.65 & 87.12 & 86.88
& 77.40 & 77.18 & 77.29 \\

\rowcolor{gray!10}
Top-1
& 97.53 & 80.35 & 88.11
& 90.86 & 75.56 & 82.51
& 88.20 & 77.51 & 82.52 \\

Uniform
& 97.29 & 82.75 & 89.43
& 92.20 & 81.87 & 86.73
& 87.93 & 77.74 & 82.53 \\

\hline
\rowcolor{lightblue}
FedSEPT
& 97.56 & 83.95 & \textbf{90.24}
& 92.59 & 83.48 & \textbf{87.80}
& 88.81 & 78.69 & \textbf{83.44} \\

\noalign{\hrule height 0.8pt}
\end{tabular}
\end{table}

\subsection{System Efficiency}
We evaluate client-side efficiency in terms of bidirectional communication per round, inference latency and FLOPs, mean training time per round, and peak GPU memory during single-batch inference. \textbf{FedSEPT-NC} denotes the first inference pass before expert text features are cached, whereas \textbf{FedSEPT-PF} replaces logit-level fusion with prompt-level fusion and repeatedly invokes the text encoder for input-dependent prompts. With $L=16$, $d=512$, $M=4$, and $r=16$, full-prompt methods transmit $2Ld=16{,}384$ floating-point parameters per round, including upload and download, while FedSEPT exchanges only $2MLr=2{,}048$ shared-factor parameters, reducing communication from $64$ to $8$ KiB in 32-bit precision. As shown in Tab.~\ref{tab:efficiency}, cached FedSEPT achieves the lowest inference latency and FLOPs at $2.346$ ms and $35.1$ GFLOPs. FedSEPT-NC incurs a one-time cost for computing expert text features, whereas FedSEPT-PF is costlier in latency, training time, and memory because it repeatedly executes the text encoder. Although FedSEPT adds modest training overhead over several single-prompt baselines, its memory usage remains comparable, confirming the efficiency of low-dimensional communication and cached logit-level fusion.

\begin{table}[t]
\centering
\caption{Client-side efficiency.}
\label{tab:efficiency}
\vspace{-2mm}
\fontsize{8}{8}\selectfont
\setlength{\tabcolsep}{2.1pt}
\renewcommand{\arraystretch}{1.2}
\begin{tabular}{l|ccccc}
\noalign{\hrule height 0.8pt}
\rowcolor{gray!25}
Method & Comm. (KiB) & Lat. (ms) & FLOPs (G) & Train (s/r) & Mem. (MB) \\
\hline\hline
PromptFL & 64.00 & 2.889 & 255.6 & 0.787 & 809.7 \\
\rowcolor{gray!10}
FedPGP & 64.00 & 2.893 & 255.6 & 0.892 & 809.7 \\
FedOTP & 64.00 & 3.395 & 476.1 & 0.927 & 821.6 \\
\rowcolor{gray!10}
FedPHA & 64.00 & 2.932 & 255.6 & 0.919 & 826.4 \\
pFedMoAP & 64.00 & 2.917 & 255.7 & 0.892 & 806.5 \\
\rowcolor{gray!10}
DP-FPL & 64.00 & 2.901 & 255.6 & 0.858 & 809.8 \\
\hline
FedSEPT-NC & 8.00 & 4.592 & 917.1 & 1.110 & 813.8 \\
\rowcolor{gray!10}
FedSEPT-PF & 8.00 & 17.340 & 255.6 & 2.798 & 3546.7 \\
\rowcolor{lightblue}
\hline
FedSEPT & 8.00 & 2.346 & 35.1 & 1.051 & 814.4 \\
\noalign{\hrule height 0.8pt}
\end{tabular}
\end{table}

\subsection{Defense Assessment Under Privacy Attacks}
\label{sec:attack}

We evaluate two privacy attacks: membership inference attacks (MIAs)~\cite{nasr2019comprehensive,bai2024miafl} and gradient inversion attacks (GIAs)~\cite{zhu2019deep,geiping2020inverting}.

\vspace{1mm}
\noindent\textbf{Membership Inference.\ }
We consider logistic-regression attackers under two threat models and report mean AUC, final AUC, and HM accuracy. 
\textbf{(1)} The \emph{update-only} attacker uses server-visible client updates. From each uploaded prompt tensor, it extracts the $\ell_2$ norm, coordinate-wise quantiles, sparsity, and temporal drift, including update norms and cosine similarity over a sliding window. A standardized logistic-regression classifier then distinguishes member samples from non-members drawn from other clients. As shown in Fig.~\ref{fig:mia_attacks}(a), the attack remains close to random guessing, with FedSEPT's final AUC ranging from $0.4487$ to $0.5021$ across privacy budgets. 
\textbf{(2)} The stronger \emph{update-and-query} attacker additionally queries the final personalized model and augments the update features with prediction statistics, including maximum confidence, prediction margin, entropy, true-class probability, and cross-entropy loss. Since these outputs depend on unshared local states, they provide a stronger leakage channel. Fig.~\ref{fig:mia_attacks}(b) shows that local DP reduces FedSEPT's final AUC from $0.9471$ without DP to $0.5521$--$0.5790$ for $\epsilon\in\{0.25,0.5,0.75,1.0\}$, while preserving HM accuracy of up to $90.68\%$. FedSEPT may exhibit slightly higher AUC than low-utility baselines because its accurate personalized predictions provide more informative query features; privacy leakage should therefore be interpreted jointly with utility.

\vspace{1mm}
\noindent\textbf{Multi-client Gradient Inversion.\ }
We audit four clients with four samples each. Since the server observes only prompt parameters before and after local training rather than per-step gradients, the attacker approximates an effective gradient as
$\hat{g}=-(\theta_{\mathrm{post}}-\theta_{\mathrm{pre}})/\eta$,
where $\eta$ is the local learning rate. Following DLG~\cite{zhu2019deep} and Inverting Gradients~\cite{geiping2020inverting}, the attacker initializes a dummy image from random noise and optimizes it with Adam so that the surrogate prompt-learner gradient matches $\hat{g}$. We use multiple restarts, retain the lowest-loss reconstruction, apply total-variation regularization, and clamp pixels to the valid range after each step.
We report pixel-level reconstruction quality using PSNR, semantic leakage using CLIP cosine similarity and top-1 agreement, and update matching using gradient cosine similarity. As shown in Tab.~\ref{tab:GIA_results}, local DP suppresses semantic leakage for FedSEPT, reducing CLIP cosine similarity from $0.529$ to $-0.013$ and top-1 agreement from $0.938$ to $0$.

\begin{figure}[t]
\centering
\includegraphics[width=0.48\textwidth]{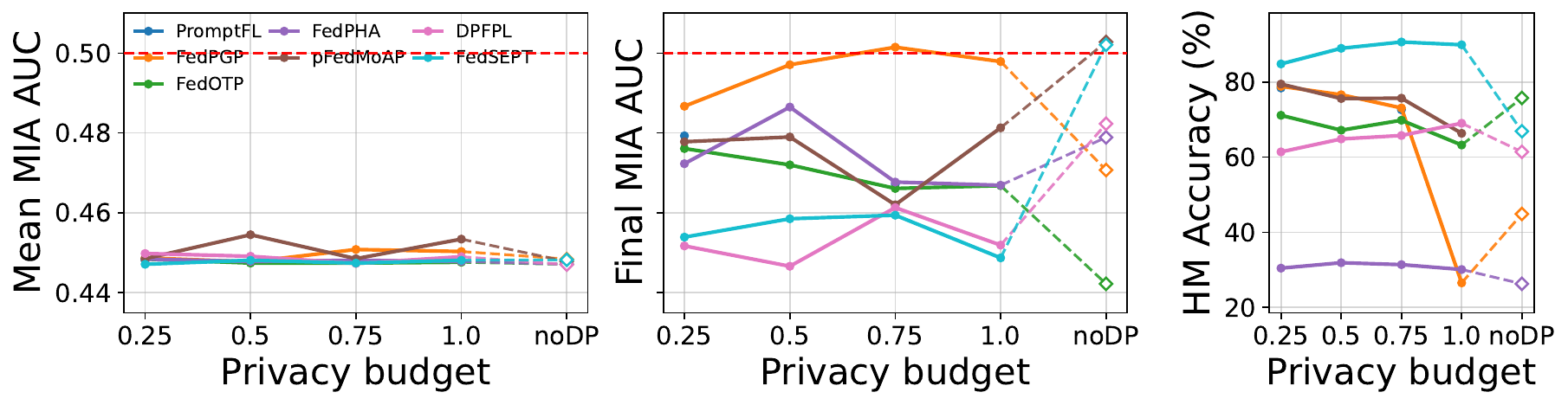}
\vspace{-4mm}
\small \textbf{(a) Update-only MIAs}

\vspace{4mm}
\includegraphics[width=0.48\textwidth]{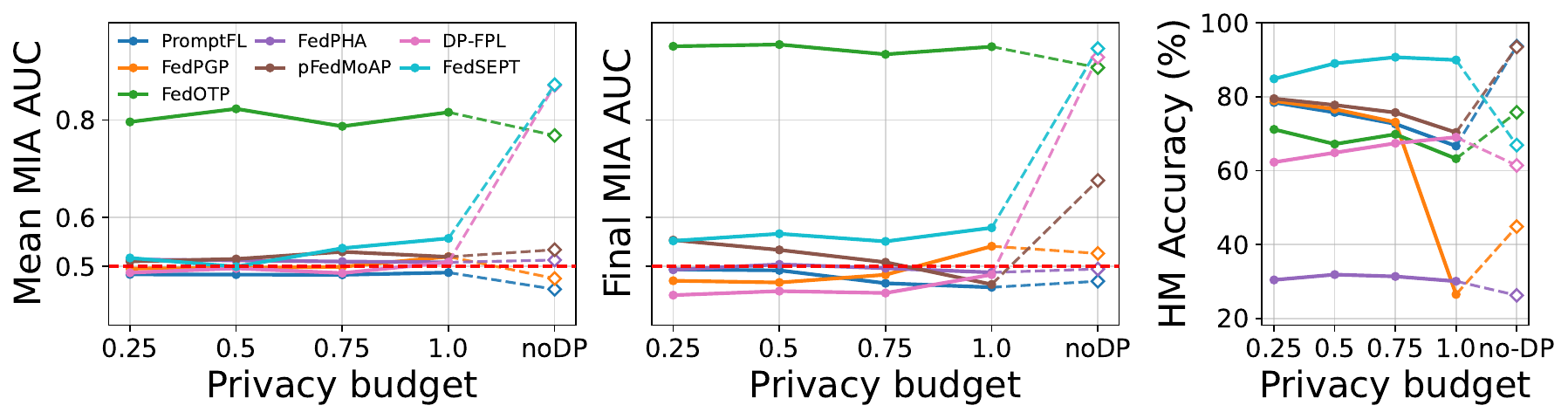}
\vspace{-4mm}
\small \textbf{(b) Update-and-query MIAs}
 
\vspace{2mm}
\caption{MIAs under two threat models on OxfordPets.}
\label{fig:mia_attacks}

\end{figure}

\begin{table}[t]
\centering
\caption{Multi-client GIAs on OxfordPets.}
\label{tab:GIA_results}
\vspace{-2mm}
\fontsize{7.8}{7}\selectfont
\setlength{\tabcolsep}{1.2pt}
\renewcommand{\arraystretch}{1.2}
\begin{tabular}{l|cccc|cccc}
\noalign{\hrule height 0.8pt}
\rowcolor{gray!25}
& \multicolumn{4}{c|}{No DP}
& \multicolumn{4}{c}{$\epsilon=1$} \\
\cline{2-9}
\rowcolor{gray!25}
\multirow{-2}{*}{Method}
& PSNR & CLIPCos & Top-1 & GradCos
& PSNR & CLIPCos & Top-1 & GradCos \\
\hline\hline
PromptFL
& 8.992 & 0.501 & 0.750 & 0.859 & 8.862 & 0.057 & 0.063 & 0.032 \\
\rowcolor{gray!10}
FedPGP
& 9.083 & 0.052 & 0.400 & 0.217 & 9.074 & 0.016 & 0.063 & 0.040 \\
FedOTP
& 9.189 & 0.497 & 0.938 & 0.790 & 8.896 & 0.034 & 0.000 & 0.026 \\
\rowcolor{gray!10}
FedPHA
& 8.808 & 0.456 & 0.500 & 0.858 & 8.680 & 0.054 & 0.125 & 0.032 \\
pFedMoAP
& 8.940 & 0.500 & 0.875 & 0.603 & 8.870 & 0.046 & 0.313 & 0.031 \\
\rowcolor{gray!10}
DP-FPL
& 8.986 & 0.459 & 0.938 & 0.362 & 8.720 & 0.066 & 0.188 & 0.039 \\
\hline
\rowcolor{lightblue}
FedSEPT
& 9.620 & 0.529 & 0.938 & 0.236 & 8.952 & -0.013 & 0.000 & 0.052 \\
\noalign{\hrule height 0.8pt}
\end{tabular}
\end{table}

\section{Conclusion}
In this paper, we proposed FedSEPT, a privacy-preserving federated prompt tuning framework for heterogeneous clients. 
By combining subspace-decomposed expert modeling with instance-aware expert fusion, FedSEPT enables privacy-efficient and robust multi-expert collaboration under local DP. 
Extensive experiments on 11 benchmarks show that, under the same privacy budget, FedSEPT achieves a better trade-off between local adaptation and global generalization than existing baselines. 
The ablation study and attack analysis further confirm its effectiveness and privacy robustness.

\newpage   
\section*{Acknowledgments} 
This work was supported by the Beijing Advanced Innovation Center for Future Blockchain and Privacy Computing, 
the National Natural Science Foundation of China (U25B2070, 62372493), 
and the Beijing Natural Science Foundation (Z230001).

\bibliographystyle{ACM-Reference-Format}
\bibliography{ref}

\clearpage
\appendix
\section{Notations} \label{app:notations}
Tab.~\ref{tab:notations} summarizes the notation used throughout this paper.

\begin{table}[hbpt]
\centering
\caption{Summary of key notations.}
\vspace{-2mm}
\label{tab:notations}
\fontsize{8}{9}\selectfont
\renewcommand{\arraystretch}{1.15}
\setlength{\tabcolsep}{1.12pt}
\begin{tabular}{l||p{6cm}}
\noalign{\hrule height 0.8pt}
\rowcolor{gray!25}
\textbf{Notation} & \textbf{Description} \\
\hline\hline
$K, T, E$ & Number of clients, communication rounds, and local epochs. \\
$\mathcal{B}, n_k, C$ & Minibatch, local sample size, and number of classes. \\
$f(\cdot), g(\cdot)$ & Frozen text encoder and image encoder. \\
$\mathbf{P}, \mathbf{P}_g, \mathbf{P}_l$ & Learnable prompt, shared global prompt, and client-specific local prompt. \\
$M, L, d, r$ & Number of experts, prompt length, embedding dimension, and low-rank dimension. \\
$\mathbf{P}_k^{m}, \mathbf{B}_0, \mathbf{R}_k$ & The $m$-th expert prompt, fixed public basis, and client-private residual. \\
$\mathbf{A}_k^{m}, \mathbf{A}, \mathbf{A}_k$ & Expert factor for expert $m$, and global and local expert-factor sets. \\
$\mathbf{v}(x), \mathbf{u}_k, \mathbf{W}_{\mathrm{gate},k}$ & Normalized visual feature, client-specific probe query, and gating matrix. \\
$\boldsymbol{\pi}_k(x), \pi_{k,m}(x), \bar{\boldsymbol{\pi}}_k$ & Client-specific expert weights, weight of expert $m$, and mean routing distribution. \\
$z_{k,c}^{m}, z_{k,c}, \mathbf{t}_{k,c}^{m}$ & Expert-specific logit, fused logit, and normalized text feature for class $c$. \\
$\mathcal{L}_{\mathrm{div}}, \mathcal{L}_{\mathrm{gate}}$ & Expert-diversity and routing regularizers. \\
$\mathcal{L}_{\mathrm{ce}}, \mathcal{L}_{\mathrm{total}}$ & Cross-entropy loss and total local objective. \\
$\epsilon, \delta$ & Privacy budget and failure probability. \\
$\sigma_k, R, S_k$ & Client noise multiplier, clipping threshold, and total DP-SGD steps. \\
\noalign{\hrule height 0.8pt}
\end{tabular}
\end{table}

\section{Experimental Details}

\subsection{Details of Datasets}
\label{app:dataset}


\subsubsection{Fine-grained recognition benchmarks}
We use five standard CLIP evaluation datasets:  \textbf{Food101}, \textbf{Caltech101}, \textbf{Flowers102}, \textbf{DTD}, and \textbf{OxfordPets}. These datasets differ in visual granularity, ranging from generic object categories to fine-grained semantic distinctions, and are therefore suitable for assessing the transferability of learned prompts under heterogeneous client distributions.

\begin{itemize}[leftmargin=1.5em, itemsep=0.2em, topsep=0.2em]

    \item \textbf{Food101}~\cite{bossard2014food} consists of 101 food categories and is challenging due to strong visual similarity across classes as well as substantial variation in presentation, ingredients, and conditions.
    
    \item \textbf{Caltech101}~\cite{10.5555/1032643.1033069} is a classic object recognition benchmark containing 101 object categories together with one background category. It covers a broad range of everyday objects and exhibits moderate intra-class variation.

    \item \textbf{Flowers102}~\cite{flower4756141} contains 102 flower categories with subtle differences in color, petal arrangement, and shape, making it a representative fine-grained classification benchmark.
    
    \item \textbf{DTD}~\cite{10.1109/CVPR.2014.461} (Describable Textures Dataset) contains 47 texture categories annotated by human-describable attributes such as \emph{striped}, \emph{bubbly}, and \emph{zigzagged}. Unlike object-centric benchmarks, DTD emphasizes appearance and texture cues.
    
    \item \textbf{OxfordPets}~\cite{parkhi12a} contains 37 categories of cat and dog breeds. Because many classes differ only in subtle local appearance, it is commonly used to evaluate fine-grained adaptation quality.

\end{itemize}

\subsubsection{Generic object classification benchmarks}
To evaluate performance on standard natural image recognition tasks, we additionally use \textbf{CIFAR-10} and \textbf{CIFAR-100}~\cite{krizhevsky2009learning}. These datasets provide complementary testbeds with relatively simple and more fine-grained label spaces, respectively.

\begin{itemize}[leftmargin=1.5em, itemsep=0.2em, topsep=0.2em]
    \item \textbf{CIFAR-10}~\cite{krizhevsky2009learning} contains 10 natural object classes, including vehicles and animals. Its compact scale and balanced original distribution make it a standard benchmark for controlled federated partitioning.

    \item \textbf{CIFAR-100}~\cite{krizhevsky2009learning} extends CIFAR-10 to 100 classes, yielding substantially finer semantic granularity and a more challenging recognition problem under heterogeneous client distributions.
\end{itemize}

\subsubsection{Multi-domain benchmarks}
To assess robustness under cross-domain variation, we further consider \textbf{PACS}, \textbf{Office31}, \textbf{OfficeHome}, and \textbf{DomainNet}. These benchmarks contain multiple visual domains with distinct styles, acquisition conditions, or abstraction levels, and are thus well suited for evaluating federated adaptation under domain shift.

\begin{itemize}[leftmargin=1.5em, itemsep=0.2em, topsep=0.2em]
    \item \textbf{PACS}~\cite{gong2012geodesic} consists of seven categories from four domains: \emph{Photo}, \emph{Art Painting}, \emph{Cartoon}, and \emph{Sketch}. The large gap between photographic and artistic domains makes it a challenging benchmark for domain generalization.
    
    \item \textbf{Office31}~\cite{saenko2010adapting} contains 31 object categories from three domains: \emph{Amazon}, \emph{DSLR}, and \emph{Webcam}. It is a widely used benchmark for cross-domain object recognition.

    \item \textbf{OfficeHome}~\cite{venkateswara2017deep} contains 65 categories from four domains: \emph{Art}, \emph{Clipart}, \emph{Product}, and \emph{Real-World}. Compared with Office31, it presents larger scale and stronger style diversity.
    
    \item \textbf{DomainNet}~\cite{peng2019moment} is a large-scale benchmark originally containing six domains; following common practice~\cite{yang2025feddda}, we use five domains in our experiments. Its diversity and scale make it one of the most challenging datasets for multi-domain evaluation.
\end{itemize}

\subsection{Details of Baseline Implementation}
\label{app:baseline_hyperparameters}
To ensure a fair and transparent comparison, we re-implemented the baseline methods—FedPGP~\cite{cui2024fedpgp}, FedOTP~\cite{li2024fedotp}, FedPHA~\cite{fang2025fedpha}, and pFedMoAP~\cite{luo2025pFedMoAP}—within a unified FL framework. All baselines strictly follow the experimental protocol detailed in Sec.~\ref{sec:Experiments}, sharing identical data partitioning, training schedules, backbone architectures, and optimization configurations. Beyond these commonalities, each method utilizes its specific default hyperparameters. 

For FedPGP, we use a rank-4 decomposition for personalized prompts. The contrastive loss weight is set to $\mu = 100$ and the temperature to $0.5$. Only the global prompt undergoes server aggregation, while the local low-rank factors remain on-device.
For FedOTP, each client maintains two sets of prompts: a global prompt for server aggregation and a private prompt for local retention. The optimal transport solver defaults to Sinkhorn, configured with \texttt{THRESH} $= 10^{-3}$, \texttt{EPS} $= 0.1$, and a maximum of 100 iterations.
For FedPHA, we set $\alpha = 1.0$, \texttt{ratio} $= 0.8$, and $\lambda_{\text{orth}} = 1$. Only the global prompt is aggregated centrally. The local prompt length is randomly sampled between 4 and 16 for each client and remains strictly personalized.
For pFedMoAP, the number of experts is set to $\min(8, \texttt{num\_users})$, utilizing \texttt{nearest} for sparse expert selection. Additional parameters are set as follows: \texttt{GATING\_HEADS} $= 4$, \texttt{GATING\_EMBED\_DIM} $= 128$, $\lambda = 0.5$, and \texttt{SCALING} $= 10.0$. During training, the server aggregates solely the prompt context, leaving the gating parameters client-specific.
For DP-FPL, we use a rank-$16$ personalized low-rank decomposition with an additional local residual. Only the global prompt is uploaded and aggregated, while all personalized components remain on-device. QR-based re-factorization is applied after each local epoch, and DP clipping and noise are imposed only on the global prompt gradients.

Unless otherwise specified, all remaining hyperparameters are kept
identical across baselines, including
\texttt{CTX\_INIT = False},
\texttt{CSC = False},
\texttt{PREC = fp16}, and
\texttt{CLASS\_TOKEN\_POSITION = end}.
We use the \textbf{Opacus} library to calibrate the noise multiplier
$\sigma_k$ for each client according to its minibatch sampling rate
$q_k=|\mathcal{B}|/n_k$ and the actual number of DP-SGD optimizer
steps $S_k$ accumulated over all rounds in which the client
participates.
For a client that performs a complete pass over its local dataset in
each participating round,
\(S_k
=
T E
\left\lceil
\frac{n_k}{|\mathcal{B}|}
\right\rceil,\)
where $T$ denotes the number of rounds in which client $k$
participates.
All experiments are conducted on two NVIDIA A100 GPUs, each with
80\,GB of memory.

\end{document}